\documentclass[a4paper,onecolumn, superscriptaddress,10pt, accepted=2021-06-25, issue=4, volume=3, shorttitle=papers/compositionality-3-4]{compositionalityarticle}
    \pdfoutput=1
\usepackage[utf8]{inputenc}
\usepackage[english]{babel}
\usepackage[T1]{fontenc}
\usepackage{hyperref}
\usepackage[numbers,sort&compress]{natbib}

\usepackage{url}
\usepackage{graphicx,caption,subcaption}
\usepackage{amsmath,amsfonts,amsthm,mathrsfs}
\usepackage{enumitem}
	\setlist[description]{leftmargin=\parindent,labelindent=\parindent}
\usepackage[dvipsnames]{xcolor}
\usepackage{tikz}
\usetikzlibrary{  cd,arrows,shapes,shapes.multipart, intersections,decorations,decorations.markings,decorations.pathreplacing,fit}
	\tikzset{thick/.style={line width=.6mm}}
	\tikzstyle{hugetensor}=[rounded rectangle,thick,draw=black,minimum width=40mm,minimum height = 3mm]
	\tikzstyle{squaretensor}=[rounded rectangle,thick,draw=black,minimum width=15mm,minimum height = 7mm]
	\tikzstyle{littletensor}=[circle,thick,draw=black,fill=red!30,inner sep=0pt,minimum size=6pt]
	\tikzstyle{tinytensor}=[circle,thick,draw=black,fill=red!30,inner sep=0pt,minimum size=6pt]

\newcommand{\<}{\langle}
\renewcommand{\>}{\rangle}

\renewcommand{\Pr}{\text{Pr}}

\DeclareMathOperator{\id}{id}
\DeclareMathOperator{\End}{End}

\DeclareMathOperator{\tr}{tr}

\newenvironment{bsmallmatrix}
  {\left[\begin{smallmatrix}}
  {\end{smallmatrix}\right]}

\DeclareMathAlphabet{\mathcal}{OMS}{cmsy}{m}{n}

\newtheorem{theorem}{Theorem}[section]
\newtheorem{corollary}{Corollary}[section]
\newtheorem{proposition}{Proposition}[section]
\newtheorem{lemma}{Lemma}[section]

\theoremstyle{definition}
\newtheorem{example}{Example}
\newtheorem{definition}{Definition}[section]

\begin{document}

\title{Language Modeling with Reduced Densities}
\date{}
\author{Tai-Danae Bradley}
\orcid{0000-0003-2995-5161}
\email{tai.danae@math3ma.com}
\affiliation{Sandbox@Alphabet, Mountain View, CA 94043, USA}
\author{Yiannis Vlassopoulos}
\email{yiannis@tunnel.tech}
\affiliation{Tunnel, New York, NY 10021, USA}
\maketitle

\begin{abstract}
This work originates from the observation that today's state-of-the-art statistical language models are impressive not only for their performance, but also---and quite crucially---because they are built entirely from correlations in unstructured text data. The latter observation prompts a fundamental question that lies at the heart of this paper: \textit{What mathematical structure exists in unstructured text data?} We put forth enriched category theory as a natural answer. We show that sequences of symbols from a finite alphabet, such as those found in a corpus of text, form a category enriched over probabilities. We then address a second fundamental question: \textit{How can this information be stored and modeled in a way that preserves the categorical structure?} We answer this by constructing a functor from our enriched category of text to a particular enriched category of reduced density operators. The latter leverages the Loewner order on positive semidefinite operators, which can further be interpreted as a toy example of entailment.
\end{abstract}

\maketitle
\tableofcontents

\section{Introduction}\label{sec:intro}
Statistical language models attempt to learn syntactic and semantic structure in language using the statistics of the language. Great progress has been made recently using neural network architectures  \cite{radford2018, vaswani2017, 2015arXiv150203509G, devlin2019}, although the problem of modeling the meaning of text---demonstrated, for instance, by a model's ability to answer questions---is still unsolved. From a mathematical perspective, a number of questions remain unanswered: What mathematical structure captures the meaning of expressions in a natural language? How much of this structure can be sufficiently detected with corpora of text? Is there a way to naturally mine abstract concepts and their interrelations? How do logic and propositional entailment arise? Even so, today's state of the art statistical language models are quite impressive for being built only from correlations in unstructured text data. This observation prompts a fundamental question that lies at the heart of this paper: \textit{What mathematical structure exists in unstructured text data?} We propose that enriched category theory provides a natural home for the answer. In particular we show that sequences of symbols from a finite alphabet, such as those found in a corpus of text, form a category enriched over probabilities. Category theory thus gives a principled means of organizing ``what goes with what'' in a corpus of text along with the statistics of the resulting expressions---precisely the information used as input to today's statistical language models. Equipped with this understanding, we then turn to another fundamental question: \textit{How can this information be stored and modeled in a way that preserves the categorical structure?} In other words, what is a representation of this mathematical structure? We propose the answer lies in a functor from our enriched category of text to a particular enriched category of linear operators. Unwinding the details is the primary goal of this paper.

We have been led to these tools from a number of independent, yet complementary, viewpoints on mathematical structure in natural language. On meaning, we take inspiration from the Yoneda lemma in category theory, which informally states that a mathematical object is completely determined by the network of relationships that object has with other objects in its environment. In this way, two objects are isomorphic if and only if they share the same network of relationships. Putting this in the context of language, we view the meaning of a word as captured by the network of ways that word fits into other expressions in the language. In the words of linguist John Firth, ``You shall know a word by the company it keeps'' \cite{firth57synopsis}. Distinguishing the meanings of words thus amounts to distinguishing the environments in which they occur \textit{in addition to}, as we put forth in this paper, the statistics of these occurrences.  Towards marrying these ideas, another viewpoint we take is that language exhibits both algebraic (or compositional) and statistical structure. Intuitively, language can be viewed as an algebra whose elements are expressions in the language and where the product of two expressions is their concatenation if the result is meaningful and is zero otherwise. Computing a representation of this algebra using statistical data in real-world text leads one directly to tensor networks \cite{PTV2017}, which are factorizations of high-dimensional tensors often used for modeling states of complex quantum many body systems \cite{orus2019,Bridgeman2017,Biamonte2019,penrose71,schollwock}, a point revisited below.

What's more, our algebraic viewpoint is not incompatible with our category theoretical one. For instance, the ``network of ways a word fits into other expressions'' may be identified with the two-sided ideal of that word. If the algebra were commutative, then we could think of language as a coordinate algebra on its spectrum whose points are the (prime) ideals of the algebra, and language would be considered as the coordinate algebra on the space of meanings (and so a translation would be a change of coordinates). A non-commutative version of this leads one to considering a category of modules, though this is a something to be studied in the future. Even so, algebraic structure alone is not sufficient to understand mathematical structure in language. Statistical features also play a vital role. Indeed, language exhibits long-distance correlations decaying with a power-law, and small-scale perturbations can propagate to all scales. For instance, changing a single word in an expression can change the entire meaning of the text: \textit{I'm going to the post office} and \textit{I'm going postal} have drastically different meanings.  Such features are also characteristic of quantum critical systems \cite{Lin_2017}, which are efficiently modeled by certain tensor networks. Inspiration also comes from a linguistic perspective, as Chomsky's linguistic theory of generative grammars leads to tensor networks as soon as one tries to make them probabilistic \cite{gallego2019language, DeGiuli_2019}.

The primary contributions of this work are theoretical, but let us emphasize a salient point about practical implementations. As alluded to above, the statistics in language observed in corpora of text resembles the same statistics observed in one-dimensional quantum critical systems, and the ground states of the latter are known to be well approximated by low rank tensor factorizations---see \cite[and references therein]{Lin_2017} as well as \cite{scaling_laws2020,gallego2019language,vidal_evenbly,PTV2017,PV2017}. With this observation in mind, we therefore work under the premise that the linear operators in our framework can be approximated efficiently by tensor networks. Indeed, the ability of tensor network methods and algorithms to efficiently handle and store data in ultra large-dimensional spaces is well-understood---reviewing it here is beyond the scope of this paper, but see \cite{roberts2019,orus2019,orus2014,Oseledets2011} and references therein. Further note that such techniques have found a increasing success in machine learning in recent years, including anomaly detection, image classification, audio classification, language modeling, and more \cite{wang2020,stoud_schwab,Stokes:2019,Miller2020,Reyes2020,roberts2019,martyn2020,Glasser:2018,Guo:2018,cohen16}. 

With this motivation in hand, the paper is organized as follows. Our model begins by viewing language as sequences from some finite vocabulary and by viewing the statistics of language as modeled by a probability distribution on this set. Section \ref{sec:densities} recalls a particular passage from classical to quantum probability by modeling any probability distribution on a finite set of sequences as a particular rank 1 density operator. Section \ref{ssec:connection} focuses on the corresponding reduced density operators, which harness valuable statistical information about the original probability distribution. We review this information in the context of language. Section \ref{sec:text_to_density} builds on this framework to give the main construction, namely the assignment to any word or phrase $s$ a particular reduced density operator $\rho_s$. As seen in Corollary \ref{corollary}, this operator has the property that it decomposes as a weighted sum of reduced density operators---one associated to each meaningful expression in the language that contains $s$---where the weights are conditional probabilities of containment. As a result, $\rho_s$ captures something of the environment, or the ``meaning,'' of $s$ in a highly principled way. An immediate consequence is that the passage $s\mapsto\rho_s$ also preserves a certain hierarchical structure that is exhibited by expressions in language. In particular, Section \ref{sec:loewner} shows that both expressions in language and their associated operators form preordered sets. The former will be given by subsequence containment of consecutive symbols, for instance: \textit{red} $\leq$ \textit{red rose}. The latter will be given by the Loewner order, where we work with reduced densities $\hat\rho_s$ not normalized to unit trace. In addition to being a map of preorders, this assignment $s\mapsto \hat\rho_s$ also preserves statistics in a compatible way. Section  \ref{sec:category} makes this statement precise using the language of category theory. We show that both the preorder of expressions $s$ and the preorder of their associated operators $\hat\rho_s$ can be equipped with the structure of categories enriched over the unit interval. The main result in Theorem \ref{theorem} concludes that the passage $s\mapsto\hat\rho_s$ amounts to an enriched functor between these enriched categories.

\subsection{Related Work}
Tensor network language models have been explored previously \cite{zhang2018,zhang2019,gallego2019language}, though to our knowledge these efforts do not seek to identify the mathematical structure in unstructured text data, nor do they ask for a faithful representation of such structure or work under the hypothesis that tensor networks are candidate representations of it. This foundational perspective is also absent from quantum language models such as \cite{basile-tamburini-2017,Sordoni2013,Li2016,chen_pan_dong2020,zhang_niu2018,li-etal-2019-cnm}. Another line of work is \cite{clark2010}, which details a categorical compositional distributional (DisCo) framework for language. The authors of \cite{Bankova_Coecke_Lewis_Marsden_2019} build upon it to describe a density operator model for entailment using the Loewner order, while \cite{piedeleu2015} use densities to model homonymy. More recently, density operators and neural methods come together in \cite{meyer-lewis-2020-modelling}, while related works on modeling entailment with densities include \cite{balkir2015,mehrnoosh2018} and references therein. Notably, DisCo models requires a choice of grammatical structure as input, which is not the case in our framework. Indeed, motivated by the success of today's statistical language models trained only on unstructured text data, we propose to let statistics as a proxy for grammar. Lastly, the primary role of the Loewner order in our work is that it instantiates the existence of an enriched functor that preserves the mathematical structure present in corpora of text. This perspective is key to the work below and is absent from the approaches listed above.

The recipe in Equation \eqref{eq:projection} for passing from a probability distribution on a set of sequences to a rank 1 density operator appears in \cite{BST2019}, where it is key to a certain tensor network generative model. The passage was further elaborated on in the context of algebraic and statistical mathematical structure in \cite{TDB}, where a preliminary discussion of this paper's framework appears in Section 3.4. We first learned of the idea to consider language as a preorder from Misha Gromov in \cite{gromov2015}. After finishing this work we noticed the same article also advocates for a ``functor'' from a ``linguistic category'' to a ``small and simple category'' \cite[p. 59]{gromov2015}. Lastly, we occasionally use tensor network diagrams to illustrate certain constructions. The diagrams are much like category theorists' string diagrams, and we assume some familiarity. For an introduction to these visual representations, see \cite{tensornetwork,Biamonte2017,orus2014,evenbly} or \cite[Section 2.2.2]{TDB}.

\subsection{Acknowledgments} The authors thank Maxim Kontsevich, Jacob Miller, and John Terilla for helpful discussions, as well as the anonymous referee for their valuable feedback.
 
\section{Modeling Probability Distributions with Density Operators}\label{sec:densities}
Let $S$ be a finite set, and let $V=\mathbb{C}^S$ denote the free complex vector space generated by $S$. Concretely, the elements of $S$ define an orthonormal basis for $V$, and we will denote the basis vector associated to $s\in S$ using the same letter $s\in V$. If an ordering is chosen so that $S=\{s_1,\ldots,s_d\}$, then by identifying each $s_i$ with the $i$th standard basis vector in $\mathbb{C}^d$ we have an isomorphism $V\cong \mathbb{C}^d$. This space has the usual inner product $\<s_i,s_j\>$, which is equal to $1$ if $i=j$ and is $0$ otherwise. Each vector $v\in V$ defines a linear functional $v^*\colon V\to\mathbb{C}$ defined by $v'\mapsto\<v,v'\>$. We denote the vector space of such linear functionals on $V$ by $V^*:=\hom(V,\mathbb{C})$. Given a finite-dimensional space $W$, we may denote elements in the tensor product $V\otimes W$ with $vw$ in lieu of $v\otimes w.$ Note that each element $wv^*$ of the tensor product $W\otimes V^*$ corresponds to a linear map $V\to W$ defined by $v'\mapsto w\<v,v'\>$. In particular, let $\End(V)$ denote the space of linear operators on $V$. Then for any unit vector $\psi\in V$, the vector $\psi\psi^*\in V\otimes V^*$ corresponds to an operator in $\End(V)$ that maps $\psi$ to itself and maps any vector orthogonal to $\psi$ to $0$. We will denote this orthogonal projection operator by $\Pr_\psi$. 

A \emph{density operator}, or simply \emph{density}, $\rho$ on a Hilbert space is a unit-trace, positive semidefinite operator. We will denote the latter property by $\rho\geq 0$. Density operators are also called \emph{quantum states} and may be thought of as the quantum analogues of classical probability distributions. Indeed, every density $\rho$ on $V= \mathbb{C}^S$ defines a probability distribution $\pi_\rho\colon S\to\mathbb{R}$ on the set $S$ by the \emph{Born rule}, where the probability of an element $s\in S$ is defined by $\pi_\rho(s):=\<\rho s, s\>.$ These values are the diagonal entries of the matrix for $\rho$ in the basis provided by $S$. They are nonnegative since $\rho$ is positive semidefinite, and their sum is 1 since $\rho$ has unit trace. Going in the other direction, any probability distribution $\pi\colon S\to\mathbb{R}$ gives rise to a density on $V$ with the property that the Born distribution induced by it coincides with $\pi$. In fact there are multiple ways to define such a density. One could consider the maximal rank diagonal operator $\sum_s\pi(s)ss^*$, whose matrix representation contains the probabilities $\pi(s)$ along its diagonal and zeros elsewhere. In what follows, however, we focus on a rank 1 density operator---namely orthogonal projection onto a particular unit vector. Concretely, consider the following unit vector in $V$,
\begin{equation}\label{eq:psi}
\psi=\sum_{s\in S}\sqrt{\pi(s)}s
\end{equation}
and let $\Pr_\psi\colon V\to V$ denote the orthogonal projection operator onto $\psi$. This unit-trace operator is positive semidefinite and satisfies $\pi_{\Pr_\psi}(s)=\<\Pr_\psi s,s\>=\sqrt{\pi(s)}^2=\pi(s)$ as claimed. The assignment $\pi\mapsto \Pr_\psi$ provides for us a key passage from classical to quantum probability whose significance is seen when the probability distribution being modeled is a joint distribution. We elaborate below.

\subsection{Understanding Reduced Densities}\label{sec:red_densities}
Given finite-dimensional vector spaces $A$ and $B$, there is an isomorphism $\End(A\otimes B)\cong \End(A)\otimes \End(B)$, and the trace defines a pair of natural linear maps $\tr_A:=\tr\otimes \id_B$ and $\tr_B:=\id_A\otimes \tr$ called \emph{partial traces} from the tensor product of the endomorphism spaces to each factor. The partial trace preserves both trace and positive semidefiniteness, and so any density operator $\rho\colon A\otimes B\to A\otimes B$ gives rise to \emph{reduced density operators}  $\rho_B:=\tr_A\rho\colon B\to B$ and $\rho_A:=\tr_B\rho\colon A\to A$. These operators may be thought of as the quantum analogues of marginal probability distributions. The analogy is especially clear when the original density $\rho$ is the orthogonal projection onto a unit vector defined by a joint probability distribution. To see this, suppose $S=X\times Y$ for finite ordered sets $X=\{x_1,\ldots,x_n\}$ and $Y=\{y_1,\ldots,y_m\}$, and let $\pi\colon S\to \mathbb{R}$ be any joint probability distribution. As before, this defines the orthogonal projection operator  $\Pr_\psi\colon \mathbb{C}^X\otimes\mathbb{C}^Y\to \mathbb{C}^X\otimes\mathbb{C}^Y$ onto the unit vector
\begin{equation}\label{eq:psi2}
\psi=\sum_{i,a}\psi_{ia}x_i y_a
\end{equation}
where $\psi_{ia}=\sqrt{\pi(x_i,y_a)}$ as in Equation \eqref{eq:psi}. Explicitly, 
\begin{equation}\label{eq:projection}
\Pr_\psi=\psi\psi^*=\left(\sum_{i,a}\psi_{ia}x_i y_a\right)\left(\sum_{j,b}\overline{\psi}_{jb}x_j^* y_b^*\right)
= \sum_{\substack{i,a\\ j,b}} \psi_{ia}\overline{\psi}_{jb}x_ix_j^*y_a y_b^*.
\end{equation}
An application of the partial trace yields the reduced density operator $\rho_Y\colon \mathbb{C}^Y\to\mathbb{C}^Y$, which has the following expression,
\begin{align}\label{eq:rho_y}
\rho_Y=\tr_X\Pr_\psi 
= \sum_{\substack{ i,a \\ j,b }} \psi_{ia}\overline{\psi}_{jb}\tr_X(x_ix_j^* y_a y_b^*)
= \sum_{\substack{ i,a \\ j,b }} \psi_{ia}\overline{\psi}_{jb}\tr(x_ix_j^*) \cdot y_a y_b^* 
=\sum_{\substack{ a,b\\ i }} \psi_{ia}\overline{\psi}_{ib}y_a y_b^*
\end{align}
where the last equality follows from $ \tr(x_ix_j^*)=\<x_j,x_i\>$, which is 1 if $i=j$ and is 0 otherwise. Notice that the $a$th diagonal entry of $\rho_Y$ is marginal probability $\pi_Y(y_a):=\sum_i \psi_{ia}\overline{\psi_{ia}}=\sum_{i}\pi(x_i,y_a)$, and so the diagonal of $\rho_Y$ recovers the marginal probability distribution $\pi_Y\colon Y\to\mathbb{R}$ obtained from the joint distribution $\pi$.
\[
    \rho_Y = 
    \begin{bmatrix}
    \pi_Y(y_1) & & & &*\\
    & \pi_Y(y_2) & & & \\
    & & \ddots && \\
    * & & & & \pi_Y(y_m)
    \end{bmatrix}
\]
The $ab$th off-diagonal entry of this matrix is $(\rho_Y)_{ab}=\sum_{i}\sqrt{\pi(x_i,y_a)\pi(x_i,y_b)}$ which is generally nonzero and measures the extent to which $y_a$ and $y_b$ have common interactions with elements in $X$. This can been expressed in terms of Bhattacharyya coefficients, which measure the proximity of two probability distributions. Unwinding this, the \emph{Bhattacharyya coefficient} for two probability distributions $p,q\colon S\to\mathbb{R}$ on a finite set $S$ is defined by $B(p,q):=\sum_s\sqrt{p(s)q(s)}$. Putting this in the context of reduced densities, each suffix $y_a\in Y$ defines a conditional probability distribution $\pi_a\colon X\to\mathbb{R}$ by $x_i\mapsto \pi(x_i|y_a)$, and so the Bhattacharyya coefficient of two conditional distributions $\pi_a$ and $\pi_b$ is equal to
\[
\sum_i\sqrt{\pi(x_i|y_a)\pi(x_i|y_b)}=\frac{1}{\sqrt{\pi_Y(y_a)\pi_Y(y_b)}}\sum_i\sqrt{\pi(x_i,y_a)\pi(x_i,y_b)} = \frac{(\rho_Y)_{ab}}{\sqrt{\pi_Y(y_a)\pi_Y(y_b)}}
\]
and the off-diagonals of the reduced density are therefore $(\rho_Y)_{ab}=\sqrt{\pi_Y(y_a)\pi_Y(y_b)}B(\pi_a,\pi_b).$ This may also be written using the \emph{Hellinger distance}, $H(p,q):=\sqrt{1-B(p,q)}$. In much the same way, the reduced density $\rho_X=\tr_Y\Pr_\psi$ contains the marginal probability distribution $\pi_X\colon X\to\mathbb{R}$ on $X$ along its diagonal and has additional nonzero off-diagonal entries. In both cases, the off-diagonals encode statistical information about subsystem interactions, and the spectral information of these reduced densities is akin to conditional probability, as it carries sufficient information to reconstruct the original state $\Pr_\psi$. This idea is described in detail in \cite[chapter 3]{TDB} and in \cite{BST2019}, where understanding and harnessing this information is key to producing a tensor network generative model. In this paper, we explore the extent to which reduced densities obtained from classical probability distributions are useful in representing words and expressions in language.

\subsection{Why Densities for Language?}\label{ssec:connection}
To motivate the connection between reduced densities and language, this section gives a few elementary, yet illuminating, observations about these operators. We begin with some terminology. Given a pair $(x,y)\in X\times Y$, refer to $x$ as the \emph{prefix} of the pair and to $y$ as the \emph{suffix}. The first observation is that two suffixes have the same image under $\rho_Y$ if and only if they share the same set of prefixes with the same probabilities (and similarly for prefixes and $\rho_X$).

\begin{proposition}\label{prop:same_image}
Let $\pi\colon X\times Y\to\mathbb{R}$ be a probability distribution and let $\psi$ be the vector given in Equation \eqref{eq:psi2}. Suffixes $y_c$ and $y_d$ satisfy $\pi(x_i,y_c)=\pi(x_i,y_d)$ for all $i$ if and only if they have the same image under $\rho_Y=\tr_X\Pr_\psi$. 
\end{proposition}
\begin{proof}
If $\pi(x_i,y_c)=\pi(x_i,y_d)$ for all $i$, then 
\[
\rho_Y(y_c)=\sum_{i,a}\sqrt{\pi(x_i,y_a)\pi(x_i,y_c)}y_a = \sum_{i,a}\sqrt{\pi(x_i,y_a)\pi(x_i,y_d)}y_a=\rho_Y(y_d).
\]
Conversely, if $\pi(x_i,y_c)\neq\pi(x_i,y_d)$ for some $i$, then $\rho_Y(y_c)\neq \rho_Y(y_d)$.
\end{proof}
\noindent Reduced densities therefore classify  suffixes (or prefixes in the case of $\rho_X$) that have the same environments and statistics within a language. For this reason, we refer to $\rho_Y(y_a)$ as the \textit{ambiance vector} for the suffix $y_a$, as suffixes with the same ambient environment have the same image under $\rho_Y.$ 

\begin{example}\label{ex:mountain}
Consider the ordered sets $X=\{\text{big, tall, cold, chilly}\}$ and $Y=\{\text{mountain, winter}\}$ and suppose $T\subseteq X\times Y$ is the four-element subset
\[T=\{(\text{big, mountain}), (\text{tall, mountain}), (\text{cold, winter}), (\text{chilly, winter}) \}.
\]
Let $\pi\colon X\times Y\to\mathbb{R}$ be the probability distribution uniformly concentrated on $T$ so that $\pi(x,y)$ is $1/4$ if $(x,y)\in T$ and is 0 otherwise. The vector in Equation \eqref{eq:psi2} is a sum of all phrases in $T$, each weighted by the square root of its probability. 
\[\psi=\tfrac{1}{2}(\text{big}\otimes \text{mountain} + \text{tall}\otimes \text{mountain} + \text{cold}\otimes \text{winter} + \text{chilly}\otimes \text{winter})\]
Following Equation \eqref{eq:rho_y}, the matrix representations of the reduced density operators obtained from orthogonal projection onto $\psi$ are given below.
\[
\rho_X = 
\frac{1}{4}
\begin{bmatrix}
1&1&0&0\\
1&1&0&0\\
0&0&1&1\\
0&0&1&1\\
\end{bmatrix}
\qquad
\rho_Y=
\frac{1}{4}
\begin{bmatrix}
2&0\\
0&2
\end{bmatrix}
\]
Observe that\footnote{Recall that elements of $Y$ are identified with standard basis vectors in $\mathbb{C}^Y\cong\mathbb{C}^2$, so $\text{mountain}=\begin{bsmallmatrix}1&0\end{bsmallmatrix}^\top$ while $\text{winter}=\begin{bsmallmatrix}0&1\end{bsmallmatrix}^\top$. Similarly, $\text{big}=\begin{bsmallmatrix}1&0&0&0\end{bsmallmatrix}^\top$, and so on.} the words \textit{mountain} and \textit{winter} share no common prefixes in $T$, and correspondingly $\rho_Y(\text{mountain})\neq\rho_Y(\text{winter})$. Intuitively, these words appear in different contexts and have different meanings, and $\rho_Y$ distinguishes them as such. On the other hand, the words \textit{big} and \textit{tall} have the same set of suffixes with identical probabilities, and  $\rho_X(\text{big})=\rho_X(\text{tall})$. Intuitively, these words have similar meanings because they appear in similar contexts, and $\rho_X$ identifies them as such.
\end{example}

This example highlights yet another connection between reduced densities and language---namely that the entries of their matrix representations have simple, combinatorial interpretations when $\pi$ is an empirical distribution. The diagonal entries of the matrix for $\rho_Y$ in Example \ref{ex:mountain} are both 2 (momentarily ignoring the factor of $1/4$) because both \textit{mountain} and \textit{winter} appear twice in the subset $T$. Importantly, the off-diagonal entries of $\rho_Y$ are zero, as \textit{mountain} and \textit{winter} have no common prefix in $T$. Similarly, the diagonals of $\rho_X$  count the number of prefixes in $T$, and the off-diagonals count the number of shared suffixes that any pair of prefixes have in common. More generally, any subset $T\subseteq X\times Y$ can be thought of as a sampling from a corpus of text and defines an empirical probability distribution $\pi\colon X\times Y\to\mathbb{R}$ by $\pi(x,y)=\frac{1}{|T|}$ if $(x,y)\in T$ and $\pi(x,y)=0$. The unit vector in Equation \eqref{eq:psi2} is then $\psi=\frac{1}{\sqrt{|T|}}\sum_{(x,y)\in T}x y$, and the $ab$th off-diagonal entry of $\rho_Y=\tr_X\Pr_\psi$ is given by $\sum_{i}\psi_{ia}\overline{\psi}_{ib}=d/|T|$ where $d$ is the number of prefixes $x_i\in X$ such that $(x_i,y_a)\in T$ and $(x_i,y_b)\in T$. A similar result holds for the reduced density $\rho_X$ on prefixes. This combinatorial observation was used in \cite{BST2019} and later elaborated on in \cite[chapter 3]{TDB}, though the application to language was not emphasized there. Let us emphasize it now. In the context of language, reduced densities neatly package the statistical information contained in their off-diagonal entries in terms of prefix-suffix interactions. As Proposition \ref{prop:same_image} has shown, this contributes to an understanding of how words fit into a language.

We take these simple observations as motivation to further explore the extent to which reduced densities arising from classical distributions can model language.  In Section \ref{sec:text_to_density} we assign to any word (or longer expression) $s$ in language a reduced density operator $\hat\rho_s$ obtained from $\Pr_\psi$, which will have the property that it contains algebraic and statistical information about the word's environment. This property, together with the Loewner order, is used in Section \ref{sec:loewner} to show that a simple concept hierarchy in language and the accompanying statistics are preserved under the passage $s\mapsto\hat\rho_s$. Section \ref{sec:category} describes how the preservation of this structure can be stated precisely in the language of category theory.

\section{Assigning Reduced Densities to Words}\label{sec:text_to_density}
To assign reduced density operators to words and expressions from a language, we start with a joint probability distribution as in Section \ref{sec:densities}. There, we considered a product of two sets, thought of as prefixes and suffixes. In this discussion, we'll consider a joint distribution on an $N$-fold product for $N\geq 2$. To this end, let $X$ be a finite ordered set consisting of the basic building blocks of a language. We'll refer to elements of $X$ as \emph{words}, though they may be characters, symbols, words, etc. Let $S=X^{N-1}\times X$ denote the set of all sequences of length $N\geq 2$. We write $S$ as a Cartesian product to obtain prefixes $(x_{i_{N-1}},\ldots, x_{i_2},x_{i_1})$ and suffixes $x_a$ so that each sequence $s\in S$ is a prefix-suffix pair $s=(x_{i_{N-1}}\cdots x_{i_2}x_{i_1},x_a)$. As shown here, the concatenation $x_{i_{N-1}}\cdots x_{i_2}x_{i_1}$ will often be used in lieu of the tuple $(x_{i_{N-1}},\ldots, x_{i_2},x_{i_1})$. Further, the indices of a prefix are labeled starting from right to left: the right-most index is $i_1$ and the left-most index is $i_{N-1}$. Consistent with Section \ref{sec:densities}, words comprising prefixes are labeled with $i,j,\ldots$ while suffixes are labeled with $a,b,\ldots$. The set of suffixes $X$ may be replaced by $X^k$ for any $k\geq 1$, though we work with $k=1$ for simplicity. Any subsequence of consecutive words in a prefix is called a \emph{phrase}, and a word is a phrase of length one.  With this setup in mind, suppose $\pi\colon S\to\mathbb{R}$ is any probability distribution and consider the unit vector $\psi=\sum_{s\in S}\psi_s s$ with $\psi_s=\sqrt{\pi(s)}$ for each $s.$ Note that $\psi$ lies in the tensor product $\mathbb{C}^S\cong V^{\otimes N-1}\otimes V$ where $V=\mathbb{C}^X$, and so since each basis vector $s$ corresponds to a tensor product $s=x_{i_{N-1}} \cdots x_{i_1}x_{a}$ we may write
\begin{equation}\label{eq:psiN}
\psi=\sum_{{i_{N-1}},\cdots,i_1,{a}}\psi_{i_{N-1}\cdots i_1a} x_{i_{N-1}} \cdots x_{i_1}x_{a}
\end{equation}
where the coefficients are the square root of probabilities $\psi_{i_{N-1}\cdots i_1a}=\sqrt{\pi(x_{i_{N-1}},\ldots,x_{i_1},x_a)}$. Now consider the rank 1 density operator on $V^{\otimes N-1}\otimes V$ given by the orthogonal projection onto $\psi,$
\[
\Pr_\psi=\psi\psi^*=\sum_{\substack{i_{N-1},\ldots, i_1,a \\ j_{N-1},\ldots, j_1,b}} \psi_{i_{N-1}\cdots i_1a}\overline{\psi}_{j_{N-1}\cdots j_1b} x_{i_{N-1}}\cdots x_{i_1}x_a x^*_{j_{N-1}}\cdots x^*_{j_1}x^*_b.
\]
Similarly as done in Equation \eqref{eq:rho_y}, tracing out the prefix subsystem gives the reduced density ${\rho_V=\tr_{V^{\otimes N-1}}\Pr_\psi\colon V\to~V}$, which has the following explicit description.
\begin{equation}\label{rho_suffix}
\rho_V=\sum_{\substack{i_{N-1},\ldots, i_1 }\\ a,b } \psi_{i_{N-1}\cdots i_1a}\overline{\psi}_{i_{N-1}\cdots i_1b} x_a x^*_b
\end{equation}
A slight modification of this expression gives rise to (unnormalized) reduced densities associated to phrases in $X^{N-1}$. Indeed, Equation (\ref{rho_suffix}) involves a sum over all prefixes, but suppose instead the sum is over all indices \textit{except} those associated to a given phrase. For example, if $N=5$ and $x_{i_2}x_{i_1}$ is a fixed phrase of length 2, consider the following operator where the indices $i_1$ and $i_2$ are fixed.
\[
\hat{\rho}_{x_{i_2}x_{i_1}}:=\sum_{{i_4,i_3}, a,b }\psi_{i_4 {i_3}  {i_2} {i_1}{a}} \overline{\psi}_{i_4 {i_3}  {i_2} {i_1}{b}}x^*_{a } x_{b}
\]
This new operator may not have unit trace, though it is still positive semidefinite. We therefore view it as \emph{the unnormalized reduced density associated to the phrase} $x_{i_2}x_{i_1}$. Informally, it is obtained by first composing the vector $x_{i_2}x_{i_1}$ with $\psi$ at the two sites directly adjacent to the suffix site, then forming the orthogonal projection onto this modified vector, and then tracing out the remaining prefix indices $i_3$ and $i_4$. The tensor network diagrams in Figure \ref{fig:rhohat} illustrate this relationship between $\psi$ and $\rho_V$ and $\hat\rho_{x_{i_2}x_{i_1}}$. This construction is summarized in Definition \ref{def:density} below and an example is given in Example \ref{ex:dog1}. Afterwards, we will show that renormalizing $\hat\rho_{x_{i_2}x_{i_1}}$ to a unit-trace operator  $\rho_{x_{i_2}x_{i_1}}$ will capture the conditional probability distribution on the set of suffixes of $x_{i_2}x_{i_1}$ in a principled way. 
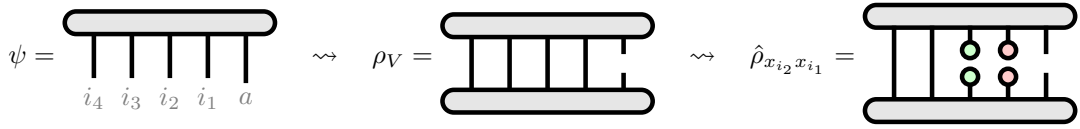
\begin{figure}[h!]
\begin{equation*}
\psi=
\begin{tikzpicture}[x=.5cm,y=1cm,baseline={(current bounding box.center)}]
    \node[hugetensor,fill=black!10, minimum width=30mm] (T1) at (0,2) {};
        \node[] (l0) at (-2,2) {};
        \node[] (l1) at (-1,2) {};
        \node[] (l2) at (0,2) {};
        \node[] (l3) at (1,2) {};
        \node[] (l4) at (2,2) {};

        \node[] (n0) at (-2,1) {{\color{gray}$i_4$}};
        \node[] (n1) at (-1,1) {{\color{gray}$i_3$}};
        \node[] (n2) at (0,1) {{\color{gray}$i_2$}};
        \node[] (n3) at (1,1) {{\color{gray}$i_1$}};
        \node[] (n4) at (2,1) {{\color{gray}$a$}};

		\draw[thick]  (l0) -- (n0);	
        \draw[thick]  (l1) -- (n1);
        \draw[thick]  (l2) -- (n2);
        \draw[thick]  (l3) -- (n3);
        \draw[thick]  (l4) -- (n4);
\end{tikzpicture}
\quad\rightsquigarrow\quad
\rho_V=
\begin{tikzpicture}[x=.5cm,y=1cm,baseline={(current bounding box.center)}]
    \node[hugetensor,fill=black!10, minimum width=30mm] (T1) at (0,2) {};
        \node[] (l0) at (-2,2) {};
        \node[] (l1) at (-1,2) {};
        \node[] (l2) at (0,2) {};
        \node[] (l3) at (1,2) {};
        \node[] (l4) at (2,2) {};

        \node[] (n0) at (-2,1) {};
        \node[] (n1) at (-1,1) {};
        \node[] (n2) at (0,1) {};
        \node[] (n3) at (1,1) {};
        \node[] (n4) at (2,1.5) {};

        \node[] (l5) at (2,1){};
        \node[] (n5) at (2,1.5){};

		\draw[thick]  (l0) -- (n0);	
        \draw[thick]  (l1) -- (n1);
        \draw[thick]  (l2) -- (n2);
        \draw[thick]  (l3) -- (n3);
        \draw[thick]  (l4) -- (n4);
        \draw[thick]  (l5) -- (n5);
    \node[hugetensor,fill=black!10, minimum width=30mm]  at (0,1) {};
\end{tikzpicture}
\quad\rightsquigarrow\quad
\hat\rho_{x_{i_2}x_{i_1}}=
\begin{tikzpicture}[x=.5cm,y=1.25cm,baseline={(current bounding box.center)}]
    \node[hugetensor,fill=black!10, minimum width=30mm] (T1) at (0,2) {};
        \node[] (l0) at (-2,2) {};
        \node[] (l1) at (-1,2) {};
        \node[] (l2) at (0,2) {};
        \node[] (l3) at (1,2) {};
        \node[] (l4) at (2,2) {};

        \node[] (n0) at (-2,1) {};
        \node[] (n1) at (-1,1) {};
        \node[] (n2) at (0,1.5) {};
        \node[] (n3) at (1,1.5) {};
        \node[] (n4) at (2,1.5) {};

        \node[] (ll2) at (0,1){};
        \node[] (nn2) at (0,1.5){};
        \node[] (ll3) at (1,1){};
        \node[] (nn3) at (1,1.5){};
        \node[] (ll4) at (2,1){};
        \node[] (nn4) at (2,1.5){};

		\draw[thick]  (l0) -- (n0);	
        \draw[thick]  (l1) -- (n1);
        \draw[thick]  (l2) -- (n2);
        \draw[thick]  (l3) -- (n3);
        \draw[thick]  (l4) -- (n4);
        \draw[thick]  (ll2) -- (nn2);
        \draw[thick]  (ll3) -- (nn3);
        \draw[thick]  (ll4) -- (nn4);

        \node[littletensor,fill=green!20,yshift=-5pt] at (nn2) {};
        \node[littletensor,fill=green!20,yshift=5pt] at (n2) {};
        \node[littletensor,fill=red!20,yshift=-5pt] at (nn3) {};
        \node[littletensor,fill=red!20,yshift=5pt] at (n3) {};

    \node[hugetensor,fill=black!10, minimum width=30mm]  at (0,1) {};
\end{tikzpicture}
\end{equation*}
\caption{A tensor network diagram illustrating the construction of the reduced densities $\rho_V$ and $\hat\rho_{x_{i_2}x_{i_1}}$ from the unit vector $\psi.$ }
\label{fig:rhohat}
\end{figure}

\begin{definition}\label{def:density}
The \emph{unnormalized reduced density} $\hat{\rho}_{x_{i_k}\cdots x_{i_1}}$  associated to a phrase $x_{i_k}\cdots x_{i_1}$ of length $k\geq 1$ is the following positive semidefinite operator on $V=\mathbb{C}^X$,
\begin{equation*}
\hat{\rho}_{x_{i_k}\cdots x_{i_1}}
:=\sum_{\substack{i_{N-1},\ldots,i_{k+1}\\a, b}} \psi_{i_{N-1} \cdots {i_1}  a} \overline{\psi}_{i_{N-1} \cdots {i_1}b}x_{a} x^*_{b}.
\end{equation*}
This operator may simply be referred to as the \emph{reduced density for}  $x_{i_k}\cdots x_{i_1}$, keeping in mind that it may not have unit trace. 
\end{definition}

\noindent As an immediate consequence, two words $x_{i_1}$ and $x_{i'_1}$ map to the same operator $\hat\rho_{x_{i_1}}=\hat\rho_{x'_{i_1}}$ if they share the same statistics in the language, that is if $\psi_{i_{N-1} \cdots {i_1}  a}=\psi_{i_{N-1} \cdots {i'_1}  a}$ for all $a$ and for all $i_{N-1},\ldots,i_2$. The same is true for more general phrases of length $k\geq 1$. Another consequence is that the reduced density for a given phrase decomposes as a sum of other reduced densities, one associated to each expression containing that phrase. Though simple, this will result reappear a number of times.

\begin{lemma}\label{lemma}
For any $1\leq k\leq N-1$ and any phrase $x_{i_k}\cdots x_{i_1}$,
\[
\hat\rho_{x_{i_k}\cdots x_{i_1}} = \sum_{i_{k+1}}\hat\rho_{x_{i_{k+1}}x_{i_k}\cdots x_{i_1}}
\]
\end{lemma}
\begin{proof}
This follows directly from the definition:
\begin{align*}
\hat\rho_{x_{i_k}\cdots x_{i_1}} 
&= 
\sum_{\substack{ i_{N-1},\ldots,i_{k+1}\\a,b}}  \psi_{i_{N-1} \cdots {i_1}  a} \overline{\psi}_{i_{N-1} \cdots {i_1}b}x_{a} x^*_{b}\\
&=\sum_{i_{k+1}}\left(\sum_{\substack{ i_{N-1},\ldots,i_{k+2}\\a,b}}  \psi_{i_{N-1} \cdots {i_1}  a} \overline{\psi}_{i_{N-1} \cdots {i_1}b}x_{a} x^*_{b}\right)\\[5pt]
&=\sum_{i_{k+1}}\hat\rho_{x_{i_{k+1}}x_{i_k}\cdots x_{i_1}}.
\end{align*}
\end{proof}
\noindent As a result, the ambiance vector associated to a word (defined below Proposition \ref{prop:same_image}) such as $x_{i_1}=\textit{dog}$ decomposes as a sum of ambiance vectors, one for each expression ending with \textit{dog}. The lemma also implies\footnote{In \cite[Section 3.4]{TDB}, an operator-sum decomposition of $\rho_V$ is used to write $\hat\rho_{x_{i_1}}$ in terms of the linear map $V^{\otimes N-1}\to V$ associated to $\psi\in V^{\otimes N-1}\otimes V$. It is equivalent to the explicit expression given in Definition \ref{def:density} above.} that the reduced density for any word $x_{i_1}$ decomposes as a sum of rank 1 operators---one associated to each phrase $x_{i_{N-1}}\cdots x_{i_1}$ that ends with $x_{i_1}$.
\begin{equation}\label{eq:decomp}
\hat\rho_{x_{i_1}} =\sum_{i_2}\hat\rho_{x_{i_2}x_{i_1}}=\sum_{i_3, i_2}\hat\rho_{x_{i_3}x_{i_2}x_{i_1}} =\cdots =\sum_{i_{N-1},\ldots,i_3,i_2}\hat\rho_{x_{i_{N-1}}\cdots x_{i_3}x_{i_2}x_{i_1}}
\end{equation}
To see that each operator $\hat\rho_{x_{i_{N-1}}\cdots x_{i_3}x_{i_2}x_{i_1}}$ has rank 1, notice that its expression in Definition \ref{def:density} does not involve a sum over prefixes. Equivalently, no edges are contracted in its tensor diagram representation, as illustrated in Figure \ref{fig:rank1}.
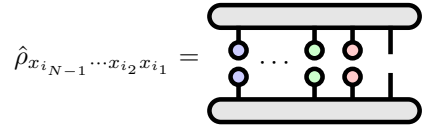
\begin{figure}[h!]
\begin{equation*}
\hat\rho_{x_{i_{N-1}}\cdots x_{i_2}x_{i_1}}=
\begin{tikzpicture}[x=.5cm,y=1.25cm,baseline={(current bounding box.center)}]
    \node[hugetensor,fill=black!10, minimum width=30mm] (T1) at (0,2) {};
        \node[] (l0) at (-2,2) {};
        \node[] (l1) at (-1,1.5) {$\cdots$};
        \node[] (l2) at (0,2) {};
        \node[] (l3) at (1,2) {};
        \node[] (l4) at (2,2) {};

        \node[] (n0) at (-2,1.5) {};
        \node[] (n2) at (0,1.5) {};
        \node[] (n3) at (1,1.5) {};
        \node[] (n4) at (2,1.5) {};

        \node[] (ll0) at (-2,1){};
        \node[] (nn0) at (-2,1.5){};
        \node[] (ll2) at (0,1){};
        \node[] (nn2) at (0,1.5){};
        \node[] (ll3) at (1,1){};
        \node[] (nn3) at (1,1.5){};
        \node[] (ll4) at (2,1){};
        \node[] (nn4) at (2,1.5){};

		\draw[thick]  (l0) -- (n0);	

        \draw[thick]  (l2) -- (n2);
        \draw[thick]  (l3) -- (n3);
        \draw[thick]  (l4) -- (n4);
        \draw[thick]  (ll0) -- (nn0);
        \draw[thick]  (ll2) -- (nn2);
        \draw[thick]  (ll3) -- (nn3);
        \draw[thick]  (ll4) -- (nn4);

 		\node[littletensor,fill=blue!20,yshift=-5pt] at (nn0) {};
        \node[littletensor,fill=blue!20,yshift=5pt] at (n0) {};
        \node[littletensor,fill=green!20,yshift=-5pt] at (nn2) {};
        \node[littletensor,fill=green!20,yshift=5pt] at (n2) {};
        \node[littletensor,fill=red!20,yshift=-5pt] at (nn3) {};
        \node[littletensor,fill=red!20,yshift=5pt] at (n3) {};

    \node[hugetensor,fill=black!10, minimum width=30mm]  at (0,1) {};
\end{tikzpicture}
\end{equation*}
\caption{A rank 1 operator illustrated as a tensor network diagram with no contracted edges.}
\label{fig:rank1}
\end{figure}
Further observe that Definition \ref{def:density} and Lemma \ref{lemma} only regard those phrases that occur adjacent to a suffix. One can also associate reduced densities to phrases that occur in \textit{any} position in a sequence (for instance, see the discussion surrounding Figure \ref{fig:bias}) and find an analogous decomposition. We omit this more general discussion to streamline the presentation. 

In the example below, we consider a toy corpus containing five phrases of length four. Each phrase will correspond to a sequence in a four-fold Cartesian product of \textit{different} sets $A\times B\times C\times D$ rather than the same set $X\times X\times X\times X$, as one might use in practice. This minor adjustment will simply keep the example tidy (for instance, it allows us to consider a $2\times 2$ matrix rather than a sparse $8\times 8$ matrix) while still illustrating the theory.

\begin{example}\label{ex:dog1}
Begin with the following ordered sets, $A=\{\text{small, big}\}$ and $B=\{\text{black, white}\}$ and $C=\{\text{dog, cat}\}$ and $D=\{\text{barks, runs}\}$, and consider the five phrases of length $N=4$ shown below on the left. Define the vector $\psi$ to be the normalized sum of these five phrases, as shown on the right.
\bigskip

\begin{minipage}{0.4\textwidth}
	\begin{description}[before={\renewcommand\makelabel[1]{\bfseries ##1}}]
		\item small black dog barks
		\item small white dog barks
		\item big black dog runs
		\item big white cat runs
		\item small black cat runs\\
	\end{description}
\end{minipage}
\begin{minipage}{0.3\textwidth}
\[\hspace{-20pt}
\psi=\frac{1}{\sqrt5}
\left(
	\begin{array}{llllllll}
	\text{small} & \otimes & \text{black} & \otimes & \text{dog} & \otimes &\text{barks} & +\\
	\text{small} & \otimes & \text{white} & \otimes & \text{dog} & \otimes &\text{barks} &  +\\
	\text{big} & \otimes & \text{black} & \otimes & \text{dog} & \otimes &\text{runs} &  +\\
	\text{big} & \otimes & \text{white} & \otimes & \text{cat} & \otimes &\text{runs} &  +\\
	\text{small} & \otimes & \text{black} & \otimes & \text{cat} & \otimes &\text{runs} &
	\end{array}
\right)
\]
\end{minipage}
\bigskip

\noindent In fact, $\psi$ is obtained from a probability distribution $\pi\colon A\times B\times C\times D\to\mathbb{R}$ as in Equation \eqref{eq:psiN}, where $\pi$ is uniformly concentrated on the subset $T\subseteq A\times B\times C\times D$ consisting of the five phrases above. Further note that $\psi$ lies in the tensor product $\mathbb{C}^A\otimes\mathbb{C}^B\otimes\mathbb{C}^C\otimes\mathbb{C}^D$, where each factor is isomorphic to $\mathbb{C}^2$. As before, elements in $A\times B\times C$ are prefixes and elements in $D$ are suffixes. Following Equation (\ref{rho_suffix}), the reduced density operator $\rho_D=\tr_{A\times B\times C}\Pr_\psi$ on the suffix subsystem is given by
\begin{align*}
\rho_D&=\sum_{(a,b,c)\in A\times B\times C}\pi(a,b,c,\text{barks})\;\text{barks}\otimes\text{barks}^* \\[5pt]
&+\sum_{(a,b,c)\in A\times B\times C}\sqrt{\pi(a,b,c,\text{barks})\pi(a,b,c,\text{runs})}\;\text{barks}\otimes\text{runs}^* \\[5pt]
&+\sum_{(a,b,c)\in A\times B\times C}\sqrt{\pi(a,b,c,\text{runs})\pi(a,b,c,\text{barks})}\;\text{runs}\otimes\text{barks}^* \\[5pt]
&+\sum_{(a,b,c)\in A\times B\times C}\pi(a,b,c,\text{runs})\;\text{runs}\otimes\text{runs}^* \\[5pt]
& = \frac{1}{5}
\begin{bmatrix}
2 & 0\\
0 & 3
\end{bmatrix}
\end{align*}
where \textit{barks} is identified with $\begin{bsmallmatrix}1\\0\end{bsmallmatrix}$ while \textit{runs} is identified with $\begin{bsmallmatrix}0\\1\end{bsmallmatrix}$. Up to the normalizing factor $1/5$, the entries of this matrix have a simple combinatorial interpretation: \textit{barks} appears twice in $T$ and \textit{runs} appears three times, and these integers are seen along the diagonal of $\rho_D$. The off-diagonals are both zero, which is the number of prefixes $(a,b,c)\in T$ that \textit{barks} and \textit{runs} have in common.  To compute the reduced density $\hat\rho_{\text{dog}}$ associated to the word \textit{dog} as in Definition \ref{def:density},  we fix $c=\textit{dog}$ and sum over pairs $(a,b)\in A\times B.$ Writing this out explicitly, 
\begin{align*}
\hat\rho_{\text{dog}}&=\sum_{(a,b)\in A\times B}\pi(a,b,\text{dog},\text{barks})\;\text{barks}\otimes\text{barks}^* \\[5pt]
&+\sum_{(a,b)\in A\times B}\sqrt{\pi(a,b,\text{dog},\text{barks})\pi(a,b,\text{dog},\text{runs})}\;\text{barks}\otimes\text{runs}^* +\\[5pt]
&+\sum_{(a,b)\in A\times B}\sqrt{\pi(a,b,\text{dog},\text{runs})\pi(a,b,\text{dog},\text{barks})}\;\text{runs}\otimes\text{barks}^* +\\[5pt]
&+\sum_{(a,b)\in A\times B}\pi(a,b,\text{dog},\text{runs})\;\text{runs}\otimes\text{runs}^* \\[5pt]
& = \frac{1}{5}
\begin{bmatrix}
2 & 0\\
0 & 1
\end{bmatrix}.
\end{align*}
Again, the entries of this matrix can be understood combinatorially. The word \textit{dog} appears three times in $T$. Of those three occurrences, it is followed by \textit{barks} twice and by \textit{runs} once, as seen along the diagonal of $\hat\rho_{\text{dog}}.$ The off-diagonals are both zero, which is the number of phrases $(a,b)$ that precede both \textit{dog barks} and \textit{dog runs}. Now suppose we fix $b=\textit{black}$ and sum over $a\in A$ alone in the calculation above. The resulting operator is the reduced density associated to the phrase \textit{black dog}. Further fixing $a=\textit{small}$ gives the reduced density for \textit{small black dog.}
\[
\hat\rho_{\text{black dog}}=
\frac{1}{5}
\begin{bmatrix}
1 & 0\\
0 & 1
\end{bmatrix}
\qquad
\hat\rho_{\text{small black dog}}=
\frac{1}{5}
\begin{bmatrix}
1&0\\
0&0
\end{bmatrix}
\]
Notably, operators associated to phrases containing the word \textit{dog} are related. By Lemma \ref{lemma} we have that $\hat\rho_{\text{dog}}=\hat\rho_{\text{black dog}} + \hat\rho_{\text{white dog}}$, where each summand further decomposes as $\hat\rho_{\text{black dog}}=\hat\rho_{\text{small black dog}}+\hat\rho_{\text{big black dog}}$ and $\hat\rho_{\text{white dog}}=\hat\rho_{\text{small white dog}}$. As a result, the reduced density for \textit{dog} can be written as a sum of reduced densities, one for every expression containing the word \textit{dog}. This pairs well with the intuition that the meaning of a word consists in all ways that word fits into expressions in the language.
\end{example}

\noindent The previous example illustrates how reduced densities obtained from a carefully chosen $\psi$ contain relevant combinatorial and statistical information about phrases in language. Additional features of these operators are given below.

\begin{proposition}\label{prop:trace}
The trace of $\hat\rho_{x_{i_k}\cdots x_{i_1}}$ for any phrase $x_{i_k}\cdots x_{i_1}$ of length $k\geq 1$ is the marginal probability of that phrase.
\end{proposition}
\begin{proof}
The trace of $\hat{\rho}_{x_{i_k}\cdots x_{i_1}}$ is the sum
\[
\sum_{{i_{N-1}}, \ldots, {i_{k+1}}, {a}}{|\<\psi, x_{i_{N-1}}\cdots x_{i_3}  x_{i_2} x_{i_1}x_{a}\>|}^2
=\sum_{{i_{N-1}}, \ldots, {i_{k+1}}, {a}} \pi(x_{i_{N-1}},\ldots,x_{i_3},  x_{i_2}, x_{i_1}, x_{a})
\]
which is  the marginal probability $\pi(x_{i_k},\cdots, x_{i_1})$.
\end{proof}
\noindent As illustrated in Example \ref{ex:dog1}, if $\pi$ is uniformly concentrated on some subset $T\subseteq X^{N-1}\times X$, then the marginal probability $\pi(x_{i_k},\cdots, x_{i_1})$ is precisely the number of times the phrase $x_{i_k}\cdots x_{i_1}$ appears within prefixes in $T$, divided by the total number of sequences $|T|.$ Here we omit subscripts on marginal probabilities to keep the notation clean. In any case, marginal probabilities provide a way to associate properly normalized density operators to phrases. Define the \emph{unit-trace reduced density for a phrase} $x_{i_k}\cdots x_{i_1}$ to be the associated reduced density divided by its trace.
\[\rho_{x_{i_k}\cdots x_{i_1}}:=\hat\rho_{x_{i_k}\cdots x_{i_1}}/\pi(x_{i_k}\cdots x_{i_1}).\]
This trace 1 operator has the property that the diagonal entries of its matrix representation in the basis given by $X$ are the conditional probabilities $\pi(x_{i_k}\cdots x_{i_1}x_{a}|x_{i_k}\cdots x_{i_1})$ for each suffix $x_{a}\in X.$ Indeed, the $a$th diagonal entry of $\rho_{x_{i_k}\cdots x_{i_1}}$ is $\<\rho_{x_{i_k}\cdots x_{i_1}}x_{a},x_{a}\>$ which is equal to
\[
\frac{1}{\pi(x_{i_k}\cdots x_{i_1})}\sum_{i_{N-1},\ldots,i_{k+1}}|\<\psi,x_{i_{N-1}}\cdots x_{i_1}x_{a}\>|^2 = \frac{\pi(x_{i_k}\cdots x_{i_1},x_{a})}{\pi(x_{i_k}\cdots x_{i_1})}
=\pi(x_{a}|x_{i_k}\cdots x_{i_1}).
\]
In this way, the operator $\rho_{x_{i_k}\cdots x_{i_1}}$ contains the probabilities that the phrase $x_{i_k}\cdots x_{i_1}$ will be continued by a given expression. 

\begin{example}\label{ex:dog2}
We resume Example \ref{ex:dog1}, where the reduced density $\hat\rho_{\text{dog}}$ is shown below, left. The trace of this operator is $\tr\hat\rho_{\text{dog}}=3/5=\pi(\text{dog})=\sum_{a,b,d}\pi(a,b,\text{dog},d)$, which is the number of times that \textit{dog} appears in the toy corpus $T$, divided by the size of $T$. The unit-trace reduced density operator is below, right.
\[\hat\rho_{\text{dog}}=
\frac{1}{5}
\begin{bmatrix}
2&0\\
0&1
\end{bmatrix}
\qquad
\rho_{\text{dog}}=
\frac{1}{3}
\begin{bmatrix}
2&0\\
0&1
\end{bmatrix}
\]
The diagonal of $\rho_{\text{dog}}$ is the conditional probability distribution on the set of suffixes $\{\text{barks, runs}\}$ conditioned on the word \textit{dog}. That is, $\<\rho_{\text{dog}}\text{barks},\text{barks}\>=2/3=\pi(\text{barks}\mid\text{dog})$, which is the conditional probability that a sequence $(a,b,c,d)\in T$ has $d=\textit{barks}$ given that $c=\textit{dog}$. Similarly, we have $\<\rho_{\text{dog}}\text{runs},\text{runs}\>=1/3=\pi(\text{runs}\mid\text{dog})$.
\end{example}

\noindent To motivate the next result, recall that the toy corpus in Example \ref{ex:dog1} contained two colors of dogs: \textit{black dog} and \textit{white dog}. As previously remarked, the meaning of the word \textit{dog} receives a contribution from the context in which it appears---including the words \textit{black} and \textit{white}---together with the statistics of those appearances. The next proposition anchors this intuition on firmer ground and can be seen as an enrichment of Lemma \ref{lemma}. It states that $\rho_{\text{dog}}$ decomposes as a weighted sum of $\rho_{\text{black dog}}$ and $\rho_{\text{white dog}}$, where the weights are conditional probabilities. 

\begin{proposition}\label{prop:decomp}
Let $1\leq k\leq N-1$. The unit-trace reduced density for any phrase $x_{i_k}\cdots x_{i_1}$ can be written as a weighted sum of unit-trace reduced densities---one for each phrase of length $k+1$ ending in $x_{i_k}\cdots x_{i_1}$---where the weights are conditional probabilities, 
\[
\rho_{x_{i_k}\cdots x_{i_1}} =   \sum_{i_{k+1}}\pi(x_{i_{k+1}} | x_{i_k}\cdots x_{i_1}){\rho}_{x_{i_{k+1}}x_{i_k}\cdots x_{i_1}}.
\]
\end{proposition}
\begin{proof}
By Lemma \ref{lemma} we have $\hat\rho_{x_{i_k}\cdots x_{i_1}} =\sum_{i_{k+1}}\hat{\rho}_{x_{i_{k+1}}x_{i_k}\cdots x_{i_1}}$ and so 
\begin{align}
\rho_{x_{i_k}\cdots x_{i_1}}
&\notag=\frac{\hat\rho_{x_{i_k}\cdots x_{i_1}}}{\pi(x_{i_k}\cdots x_{i_1})}
= \sum_{i_{k+1}}\frac{\pi(x_{i_{k+1}}x_{i_k}\cdots x_{i_1})}{\pi(x_{i_k}\cdots x_{i_1})}\frac{\hat\rho_{x_{i_{k+1}}x_{i_k}\cdots x_{i_1}}}{\pi({x_{i_{k+1}}x_{i_k}\cdots x_{i_1}})}\\
&\notag=\sum_{i_{k+1}}\pi(x_{i_{k+1}}|x_{i_k}\cdots x_{i_1})\rho_{x_{i_{k+1}}x_{i_k}\cdots x_{i_1}}.
\end{align}
\end{proof}
\noindent So Proposition \ref{prop:decomp} relates the unit-trace reduced density $\rho_{x_{i_k}\cdots x_{i_1}}$ to all phrases of length $k+1$ that end with the given phrase $x_{i_k}\cdots x_{i_1}$. The proof of the following corollary gives the analogous statement for phrases of length $N-1$.

\begin{corollary}\label{corollary}
The unit-trace reduced density for a word $x_{i_1}$ decomposes as a weighted sum of rank 1 unit-trace reduced densities---one for each phrase of length $N-1$ that ends in $x_{i_1}$---where the weights are conditional probabilities, 
\begin{equation}\label{eq:rhoxi1}
\rho_{x_{i_1}} = \sum_{{i_{N-1}},\ldots,{i_2}}\pi(x_{i_{N-1}}\cdots x_{i_2} | x_{i_1})\rho_{x_{i_N-1}\cdots x_{i_2}x_{i_1}}.
\end{equation}
\end{corollary}
\begin{proof}
For any phrase $x_{i_k}\cdots x_{i_1}$ of any length $k\geq 1$, a repeated application of Proposition \ref{prop:decomp} shows that the unit-trace reduced density of the phrase has the following decomposition.
\begin{align*}
\rho_{x_{i_k}\cdots x_{i_1}}
&=\sum_{i_{k+1}}\pi(x_{i_{k+1}} | x_{i_k}\cdots x_{i_1}){\rho}_{x_{i_{k+1}}x_{i_k}\cdots x_{i_1}}\\
&=\sum_{i_{k+2},i_{k+2}}\pi(x_{i_{k+2}}x_{i_{k+1}}\mid x_{i_k}\cdots x_{i_1})\rho_{x_{i_{k+2}}x_{i_{k+1}}x_{i_k}\cdots x_{i_1}}\\
&=\quad\vdots\\
&=\sum_{i_{N-1},\ldots,i_{k+1}}\pi(x_{i_{N-1}}\cdots x_{i_{k+1}}\mid x_{i_k}\cdots x_{i_1})\rho_{x_{i_{N-1}}\cdots x_{i_k}\cdots x_{i_1}}
\end{align*}
In particular, the unit-trace reduced density associated to a given word $x_{i_1}$ can be written as the following sum, which can be seen as an enrichment of Equation \eqref{eq:decomp}.
\[
\rho_{x_{i_1}} = \sum_{{i_{N-1}},\ldots,{i_2}}\pi(x_{i_{N-1}}\cdots x_{i_2} | x_{i_1})\rho_{x_{i_N-1}\cdots x_{i_2}x_{i_1}}
\]
Recall that each $\rho_{x_{i_{N-1}}\cdots x_{i_3}x_{i_2}x_{i_1}}$ is a scalar multiple of $\hat\rho_{x_{i_{N-1}}\cdots x_{i_3}x_{i_2}x_{i_1}}$, and the latter has rank 1 as its explicit expression in terms of Definition \ref{def:density} does not involve a sum over prefixes.
\end{proof}

\noindent Proposition \ref{prop:decomp} and Corollary \ref{corollary} model the idea that the meaning of a phrase is contained in the totality of expressions that contain it, together with the statistics of those occurrences. Notably, the decomposition in Corollary \ref{corollary} is not unlike a ``probabilistic spectral decomposition.'' Indeed every self-adjoint operator, including $\rho_{x_{i_1}}$, can be written as a weighted sum of projection operators corresponding to eigenvectors. The previous corollary gives an analogous decomposition where the projections correspond not to eigenvectors but rather to vectors associated to phrases that end with the word $x_{i_1}$. Likewise, the weights are not eigenvalues, but are rather conditional probabilities of phrases given that their last word is $x_{i_1}$. Alternatively, the decomposition in Equation \eqref{eq:rhoxi1} is reminiscent of a generalized measurement in quantum mechanics---a collection of positive semidefinite operators whose sum is the identity operator. Though rather than a partition of unity, we have a partition of the observable $\rho_{x_{i_1}}$.

\begin{example}\label{ex:dog3}
Let us again resume the discussion from Example \ref{ex:dog1}, where the unit-trace reduced densities associated to \textit{black dog} and \textit{white dog} are found to be
\[
\rho_{\text{black dog}}=
\frac{1}{2}
\begin{bmatrix}
1&0\\ 0&1
\end{bmatrix}
\qquad
\rho_{\text{white dog}}=
\begin{bmatrix}
1&0\\ 0&0
\end{bmatrix}.
\]
Recalling the toy corpus of that example, the word \textit{dog} appears three times. Of those three occurrences, it is preceded by \textit{white} once and by \textit{black} twice. Conditional probabilities are therefore $\pi(\text{white}\mid\text{dog})=1/3$ and $\pi(\text{black}\mid\text{dog})=2/3$, and indeed $\rho_{\text{dog}}$ has the following decomposition,
\[
\frac{1}{3}
\begin{bmatrix}
2&0\\0&1
\end{bmatrix}=
\rho_{\text{dog}}=
\pi(\text{black}\mid\text{dog})\rho_{\text{black dog}} + 
\pi(\text{white}\mid\text{dog})\rho_{\text{white dog}}
\]
where the first equality was verified in Example \ref{ex:dog2}. Compare this decomposition for $\rho_\text{dog}$ with the unnormalized analogue $\hat\rho_{\text{dog}}=\hat\rho_{\text{white dog}} + \hat\rho_{\text{black dog}}$ derived in Example \ref{ex:dog1}. A computation similar to that in Example \ref{ex:dog3} shows that $\rho_\text{dog}$ further decomposes into a sum of rank 1 operators,
\begin{align*}\rho_{\text{dog}}
&= \pi(\text{small black}|\text{dog})\rho_{\text{small black dog}}\\[5pt]
&+ \pi(\text{small white}|\text{dog})\rho_{\text{small white dog}}\\[5pt]
&+ \pi(\text{big black}|\text{dog})\rho_{\text{big black dog}} \\[5pt]
&=\frac{1}{3}\begin{bmatrix}1&0\\0&0\end{bmatrix} + \frac{1}{3}\begin{bmatrix}1&0\\0&0\end{bmatrix} + \frac{1}{3}\begin{bmatrix}0&0\\1&0\end{bmatrix}.
\end{align*}
Compare this with the unnormalized analogue $\hat\rho_{\text{dog}}=\hat\rho_{\text{small black dog}}+\hat\rho_{\text{big black dog}}+\hat\rho_{\text{small white dog}}$ derived in Example \ref{ex:dog1}.
\end{example}

As we'll see in Section \ref{sec:loewner}, the passage from $x_{i_{k}}\cdots x_{i_1}$ to $\hat\rho_{x_{i_{k}}\cdots x_{i_1}}$ together with the decompositions in Lemma \ref{lemma} and Proposition \ref{prop:decomp} pave the way for modeling a simple concept hierarchy in language. But first, we close this discussion with the observation that our reduced densities have an obvious left-right bias, which may be undesirable. One way to avoid this bias is simply to leave the $i_{N-1}$ and $a$th indices open, which gives an operator on $V\otimes V$ rather than on $V$ alone. Figure \ref{fig:bias} illustrates such an operator corresponding to a given phase $x_{i_3}x_{i_2}x_{i_1}$ of length three. But whether or not the $N-1$st sites are kept open, let us make a final remark. Suppose for the moment that $x_{j_2}x_{j_1}$ is a phrase of length two, and suppose a corpus of text is given where both $x_{i_3}x_{i_2}x_{i_1}$ and $x_{j_2}x_{j_1}$ appear simultaneously in a longer expression. Then the decompositions of the densities $\hat\rho_{x_{i_3}x_{i_2}x_{i_1}}$ and $\hat\rho_{x_{j_2}x_{j_1}}$ in the sense of Equation \eqref{eq:decomp} will contain a common summand. For example, if a corpus of text contains the expression \textit{I walked my tiny toy poodle at the new park}, then one will find that both $\hat\rho_{\text{tiny toy poodle}}$ and $\hat\rho_{\text{new park}}$ can be written as a sum of operators, both of which will include $\hat\rho_{\text{I walked my tiny toy poodle at the new park}}$ as a summand. This realizes the intuition that if two different phrases are both included in a larger expression, then there is a relationship between their meanings. 
\begin{figure}[h!]
\begin{equation*}
\hat\rho_{x_{i_3}x_{i_2}x_{i_1}}=
\begin{tikzpicture}[x=.5cm,y=1.25cm,baseline={(current bounding box.center)}]
    \node[hugetensor,fill=black!10, minimum width=50mm] (T1) at (0,2) {};
        \node[] (k2) at (-4,2) {};
        \node[] (k1) at (-3,2) {};
        \node[] (l0) at (-2,2) {};
        \node[] (l1) at (-1,2) {};
        \node[] (l2) at (0,2) {};
        \node[] (l3) at (1,2) {};
        \node[] (l4) at (2,2) {};
        \node[] (l5) at (3,2) {};
        \node[] (l6) at (4,2) {};

        \node[] (kk2) at (-4,1){};
        \node[] (ll1) at (-1,1){};
        \node[] (ll2) at (0,1){};
        \node[] (ll3) at (1,1){};
        \node[] (ll6) at (4,1){};

        \node[] (m2) at (-4,1.5) {};
        \node[] (m1) at (-3,1) {};
        \node[] (n0) at (-2,1) {};
        \node[] (n1) at (-1,1.5) {};
        \node[] (n2) at (0,1.5) {};
        \node[] (n3) at (1,1.5) {};
        \node[] (n4) at (2,1) {};
        \node[] (n5) at (3,1) {};
        \node[] (n6) at (4,1.5) {};

        \node[] (mm2) at (-4,1.5){};
        \node[] (nn1) at (-1,1.5){};
        \node[] (nn2) at (0,1.5){};
        \node[] (nn3) at (1,1.5){};
        \node[] (nn6) at (4,1.5){};

        \draw[thick]  (k2) -- (m2);
        \draw[thick]  (k1) -- (m1);
		\draw[thick]  (l0) -- (n0);	
        \draw[thick]  (l1) -- (n1);
        \draw[thick]  (l2) -- (n2);
        \draw[thick]  (l3) -- (n3);
        \draw[thick]  (l4) -- (n4);
        \draw[thick]  (l5) -- (n5);
        \draw[thick]  (l6) -- (n6);

        \draw[thick]  (kk2) -- (mm2);
        \draw[thick]  (ll1) -- (nn1);
        \draw[thick]  (ll2) -- (nn2);
        \draw[thick]  (ll3) -- (nn3);
        \draw[thick]  (ll6) -- (nn6);

        \node[littletensor,fill=green!20,yshift=-5pt] at (nn2) {};
        \node[littletensor,fill=green!20,yshift=5pt] at (n2) {};
        \node[littletensor,fill=red!20,yshift=-5pt] at (nn3) {};
        \node[littletensor,fill=red!20,yshift=5pt] at (n3) {};
        \node[littletensor,fill=blue!20,yshift=-5pt] at (nn1) {};
        \node[littletensor,fill=blue!20,yshift=5pt] at (n1) {};

    \node[hugetensor,fill=black!10, minimum width=50mm]  at (0,1) {};
\end{tikzpicture}
\qquad\qquad
\hat\rho_{x_{j_2}x_{j_1}}=
\begin{tikzpicture}[x=.5cm,y=1.25cm,baseline={(current bounding box.center)}]
    \node[hugetensor,fill=black!10, minimum width=50mm] (T1) at (0,2) {};
        \node[] (k2) at (-4,2) {};
        \node[] (k1) at (-3,2) {};
        \node[] (l0) at (-2,2) {};
        \node[] (l1) at (-1,2) {};
        \node[] (l2) at (0,2) {};
        \node[] (l3) at (1,2) {};
        \node[] (l4) at (2,2) {};
        \node[] (l5) at (3,2) {};
        \node[] (l6) at (4,2) {};

        \node[] (kk2) at (-4,1){};
        \node[] (ll4) at (2,1){};
        \node[] (ll5) at (3,1){};
        \node[] (ll6) at (4,1){};

        \node[] (m2) at (-4,1.5) {};
        \node[] (m1) at (-3,1) {};
        \node[] (n0) at (-2,1) {};
        \node[] (n1) at (-1,1) {};
        \node[] (n2) at (0,1) {};
        \node[] (n3) at (1,1) {};
        \node[] (n4) at (2,1.5) {};
        \node[] (n5) at (3,1.5) {};
        \node[] (n6) at (4,1.5) {};

        \node[] (mm2) at (-4,1.5){};
        \node[] (nn4) at (2,1.5){};
        \node[] (nn5) at (3,1.5){};
        \node[] (nn6) at (4,1.5){};

        \draw[thick]  (k2) -- (m2);
        \draw[thick]  (k1) -- (m1);
		\draw[thick]  (l0) -- (n0);	
        \draw[thick]  (l1) -- (n1);
        \draw[thick]  (l2) -- (n2);
        \draw[thick]  (l3) -- (n3);
        \draw[thick]  (l4) -- (n4);
        \draw[thick]  (l5) -- (n5);
        \draw[thick]  (l6) -- (n6);

        \draw[thick]  (kk2) -- (mm2);
        \draw[thick]  (ll4) -- (nn4);
        \draw[thick]  (ll5) -- (nn5);
        \draw[thick]  (ll6) -- (nn6);
       
        \node[littletensor,fill=yellow!20,yshift=-5pt] at (nn4) {};
        \node[littletensor,fill=yellow!20,yshift=5pt] at (n4) {};
        \node[littletensor,fill=black!30,yshift=-5pt] at (nn5) {};
        \node[littletensor,fill=black!30,yshift=5pt] at (n5) {};

    \node[hugetensor,fill=black!10, minimum width=50mm]  at (0,1) {};
\end{tikzpicture}
\end{equation*}
\caption{Associating reduced densities to phrases while leaving the first and last indices open.}
\label{fig:bias}
\end{figure}
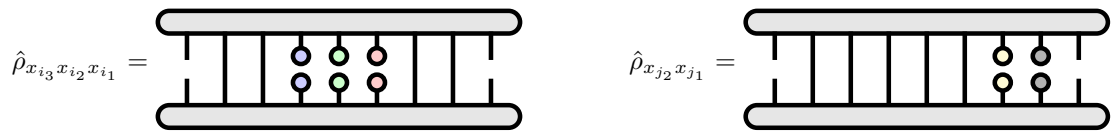

\section{Preserving a Preorder Structure}\label{sec:loewner}
In this section we  define a preorder on sequences and show that it is preserved under the assignment of densities to words described in the previous section. We also show it may be interpreted as a toy example of entailment. Let us begin with a definition. Given positive semidefinite operators $\rho$ and $\rho'$ on a fixed Hilbert space, write $\rho\geq \rho'$ if $\rho-\rho'$ is positive semidefinite. This defines a partial order on the set of such operators called the \emph{Loewner order}. In the context of language, observe that for any phrase $x_{i_k}\cdots x_{i_1}$ and for any word $x_{i_{k+1}}$, Lemma \ref{lemma} implies 
\begin{equation}\label{eq:order1}
\hat\rho_{x_{i_k}\cdots x_{i_1}}\geq \hat\rho_{x_{i_{k+1}}x_{i_k}\cdots x_{i_1}}
\end{equation}
since the difference of these operators is a sum of positive semidefinite operators. For instance, in Example \ref{ex:dog1} it was shown that $\hat\rho_\text{dog}=\hat\rho_\text{black dog}+\hat\rho_\text{white dog}$ from which it follows that $\hat\rho_\text{dog}\geq \hat\rho_\text{black dog}$ and $\hat\rho_\text{dog}\geq \hat\rho_\text{white dog}.$ Taking this a step further, we have
\begin{equation}\label{eq:compare_dogs}
\hat\rho_\text{dog}\geq \hat\rho_\text{black dog}\geq \hat\rho_\text{small black dog}.
\end{equation}
The Loewner order thus models the notion that \textit{dog} is a more general concept than \textit{black dog}, which is more general than \textit{small black dog}. Just as essential, though, are the likelihood or entailment strengths of these implications. Proposition  \ref{prop:decomp} naturally suggests conditional probabilities as a candidate for such a measurement.  Indeed, the proposition promotes Inequality \eqref{eq:order1} to the following ``enriched'' version,
\begin{equation}\label{eq:order2}
\rho_{x_{i_k}\cdots x_{i_1}}\geq \pi(x_{i_{k+1}}\mid x_{i_k}\cdots x_{i_1})\rho_{x_{i_{k+1}}x_{i_k}\cdots x_{i_1}}.
\end{equation}
From Example \ref{ex:dog3}, for instance, we find that 
\[\rho_\text{dog}\geq \pi(\text{black}\mid\text{dog})\rho_\text{black dog}\geq \pi(\text{small black}\mid\text{dog})\rho_\text{small black dog}.\]

\noindent We have therefore defined a mapping from phrases to reduced densities with the property that if a phrase $s'$ contains a phrase $s$ of smaller length, then the corresponding operators satisfy $\rho_{s}\geq \pi(s'~|~s) \rho_{s'}$, or in the unnormalized case, $\hat\rho_{s}\geq \hat\rho_{s'}$. 
In short, the assignment $s\mapsto\hat\rho_s$ preserves a certain hierarchy that exists in the domain, namely subsequence containment. Said differently, expressions in a language form a preordered set (that is, a set equipped with a relation $\leq$ that is reflexive and transitive), and the assignment $s\mapsto\hat\rho_s$ respects the preorder.

\subsection{Language as a Preordered Set}
The opening remarks of Section \ref{sec:intro} shared the perspective that the meaning of a word or phrase is determined by its environment---namely, the network of ways it is contained within expressions in a language together with the statistics of those occurrences. Putting this Yoneda-lemma-like approach together with the sequential nature of language, we model the inclusion of phrases by subsequence containment. Explicitly, let $X$ be a finite set of words, thought of as the atomic vocabulary of a language. For $N\geq 1$, let $L$ denote the subset of $\sqcup_{k=1}^{N-1}X^k\times X$ consisting of sequences $s=(x_{i_k},\ldots,x_{i_1},x_a)$ of all lengths $k\leq N-1$ such that the concatenation $x_{i_k}\cdots x_{i_1}x_a$ is a meaningful expression in the language. We will refer to elements of $L$ as \textit{phrases}, as usual writing $x_{i_k}\cdots x_{i_1}x_a$ in lieu of $(x_{i_k},\ldots,x_{i_1},x_a)$. If $X$ is the set of English words, for example, then some phrases in $L$ include \textit{dog, black dog, tall mountain}, and \textit{iced tea on a hot summer day.} To compare phrases $s,s'\in L$, write $s\leq s'$ if $s'$ contains $s$ as a subsequence. As a simple example, we write $\textit{dog}\leq\textit{small black dog}.$ Next, observe that for any $s\in L$ one has $s\leq s$. Moreover, for any $s,s',s''\in L$ if $s\leq s'$ and $s'\leq s''$, then $s''$ contains $s'$---and hence $s$---as a subsequence, and so $s\leq s''.$ This proves the following proposition. 
\begin{proposition}
The set $L$ equipped with the relation $\leq$ is a preordered set.
\end{proposition}

\noindent 
Some examples of comparable phrases in English include the following.
\begin{align*}
\text{I climbed} &\leq \text{I climbed the tall mountain} \\
\text{a hot summer} &\leq \text{iced tea on a hot summer day}\\
\text{dog}&\leq\text{black dog}\leq \text{small black dog}
\end{align*}
The similarity between $\textit{dog}\leq \textit{black dog} \leq \textit{small black dog}$ and the nearly identical string of inequalities in \eqref{eq:compare_dogs} reveals an unmistakable correspondence between our preorder on language $L$ and the Loewner order, which is a preorder on positive semidefinite operators. This correspondence is stated precisely in Proposition \ref{prop:poset_map} below. Let us recall the setup first. Sequences in $X^N\cong X^{N-1}\times X$ are considered as prefix-suffix pairs; the initial ingredient is a probability distribution $\pi\colon X^{N-1}\times X\to\mathbb{R}$, which defines the unit vector $\psi\in V^{\otimes N-1}\otimes V$ in Equation \eqref{eq:psiN} where $V=\mathbb{C}^X$; the vector gives rise to reduced densities of the form $\hat\rho_s$ or $\rho_s$, which correspond to (sub)sequences $s\in X^{N-1}$; and these reduced densities operate on the Hilbert space $V$ generated by suffixes. Importantly, the assignment $s\mapsto\hat\rho_s$ concerns only those phrases in $L$ of the form $s=x_{i_k}\cdots x_{i_1}$ for some $1\leq k\leq N-1$ that appear adjacent to a suffix, as indicated below:
\[(x_{i_{N-1}},\ldots,\overbrace{x_{i_k},\ldots,x_{i_1}}^{s},x_a).\]
In what follows, we use the calligraphic font $\mathcal{L}\subseteq L$ to denote the subset of all such $s.$

\begin{proposition}\label{prop:poset_map}
Let $(\text{Pos}(V),\leq)$ denote the set of positive semidefinite operators on $V=\mathbb{C}^X$ equipped with the Loewner order, and let $\psi\in V^{\otimes N-1}\otimes V$ be the unit vector in Equation \eqref{eq:psiN}. The function $(\mathcal{L},\leq)\to (\text{Pos}(V),\leq)$ defined by $s\mapsto\hat\rho_s$ described in Definition \ref{def:density} is order-reversing. That is, $\hat\rho_s\geq \hat\rho_{s'}$ whenever $s\leq s'$.
\end{proposition}
\begin{proof}
Suppose $s\leq s'$ so that $s=x_{i_k}\cdots x_{i_1}$ and $s'=x_{i_m}\cdots x_{i_{k+1}}s$ for some $1\leq k\leq m\leq N-1$. Lemma \ref{lemma} implies that $\hat\rho_s\geq \hat\rho_{s'}$. 
\end{proof}

\noindent Observe that properties of the mapping $s\mapsto\hat\rho_s$ depend on the probability distribution $\pi$ used to define the vector $\psi$. In the toy scenario of Example \ref{ex:dog1}, for instance, the mapping is not ``full'' in the category theoretical sense since one finds that $\hat\rho_{\text{black cat}}\leq \hat\rho_{\text{dog}}$ under the Loewner order, as $\hat\rho_{\text{black cat}}=\tfrac{1}{5}\begin{bsmallmatrix}0&0\\0&1\end{bsmallmatrix}$, and yet $\textit{dog}\not\leq\textit{black cat}$ in  $\mathcal{L}$. This may not be the case, however, for different choices of $\pi.$ On the other hand, $\hat\rho_{\text{black cat}}\leq \hat\rho_{\text{dog}}$  \textit{does} imply a relationship between the sets of possible suffixes for these two phrases---see the discussion in Section \ref{ssec:conclusion}. Either way, Proposition \ref{prop:poset_map} shows that our preorder on language models a simple form of hierarchy which is preserved under the passage to linear algebra described in Section \ref{sec:text_to_density}. Thinking back to the discussion of meaning, it is now simple to identify the context, or environment, of a word. It is simply upper closure. For instance, the set of all expressions that contain \textit{dog} is given by\footnote{In the language of category theory, the upper closure of \textit{dog} is the image of \textit{dog} under the Yoneda embedding $L^{\text{op}}~\to~UL$, where $UL$ denotes all upward closed subsets of $L$ ordered by inclusion, and where the preorders $L$ and $UL$ are viewed as categories enriched over truth values.} $\uparrow\{\text{dog}\}:=\{s\in L\mid \text{dog}~\leq~s~\}$. But for simplicity, let us restrict attention to only those expressions $x_{i_{k}}\cdots x_{i_2}x_{i_1}\in \mathcal{L}\subseteq L$ where the last word is fixed at $x_{i_1}=\text{dog}$. In this case, the upper closure of \textit{dog} consists of all phrases in $\mathcal{L}$ of length at most $N-1$ that contain \textit{dog} as the last word. In other words, $\uparrow\{\text{dog}\}$ is equal to
\[
\{x_{i_2}\text{dog}\mid x_{i_2}\in X\}
\sqcup \{x_{i_3}x_{i_2}\text{dog}\mid x_{i_3}, x_{i_2}\in X\}
\sqcup \cdots \sqcup
\{x_{i_{N-1}}\cdots x_{i_3}x_{i_2}\text{dog}\mid x_{i_{N-1}},\ldots, x_{i_3}, x_{i_2}\in X\}.
\]
As implied early on by Equation \eqref{eq:decomp}, the passage $\mathcal{L}\to\text{Pos}(V)$ gives rise to an analogous decomposition of words associated to these expressions:
\[
\hat\rho_{\text{dog}} =\sum_{i_2}\hat\rho_{x_{i_2}\text{dog}}=\sum_{i_3, i_2}\hat\rho_{x_{i_3}x_{i_2}\text{dog}} =\cdots =\sum_{i_{N-1},\ldots,i_3,i_2}\hat\rho_{x_{i_{N-1}}\cdots x_{i_3}x_{i_2}\text{dog}}.
\]
In this way, something of the ``meaning'' of \textit{dog}---that is, the environment in which the word appears---is neatly packed into the single operator $\hat\rho_{\text{dog}}$. But as previously noted, the statistics accompanying these appearances are also essential for capturing meaning. For instance, the conditional probability of \textit{black} given that the next word is \textit{dog} contributes to the latter's meaning. We therefore wish to ``decorate'' the preorder structure $\leq$ with probabilities in such a way that the order-preserving map $\mathcal{L}\to \text{Pos}(V)$ retains knowledge of these conditional probabilities.
\[\text{dog}\leq \text{black dog}
\quad\rightsquigarrow\quad
\text{dog} \overset{\pi(\text{black}|\text{dog})}{\leq} \text{black dog}\]
These conditional probabilities arise naturally in the discussion on unit-trace reduced densities as in Inequality \eqref{eq:compare_dogs}. For any $s$ and $s'$ in $\mathcal{L}$, the containment $s\leq s'$ implies $\rho_s\geq \pi(s'\mid s) \rho_{s'}$. But unit trace operators are not comparable under the Loewner order, and so we  look for another way to  incorporate probabilities with the preorder structure. We needn't look far, however. Category theory provides a natural setting for these ideas \cite{riehl2017category,leinster2014basic,awodey2010category}. Indeed, every preordered set is an example of a category, and the function in Proposition \ref{prop:poset_map} is a contravariant functor. The ability to ``decorate'' $\leq$ with probabilities is the notion behind enriched category theory \cite{kelly1982basic}, \cite[Appendix A]{Elliott2017OnTF}. This is made precise in Theorem \ref{theorem} below, which generalizes Proposition \ref{prop:poset_map} by incorporating probabilities in the desired way. We unwind this in the next section.
  
\section{Language as a Category Enriched Over Probabilities}\label{sec:category}
The full machinery of enriched category theory is not needed for our framework, and so the discussion will be kept simple. Indeed, the only categories being considered are preordered sets, and the only category we wish to enrich over is a particular symmetric monoidal preorder.
\begin{definition}
A \emph{symmetric monoidal preorder} $(P,\leq,\cdot,1)$ is a preorder $(P,\leq)$ together with
	\begin{itemize}
		\item an element $1\in P$ called the \emph{monoidal unit}
		\item a function $\cdot\colon P\times P\to P$ called the \emph{monoidal product}.
		\end{itemize}
Moreover these data must satisfy the following properties (where we write $pq$ for $p\cdot q)$:
	\begin{itemize}
		\item for all $p,p',q,q'\in P$, if $p\leq p'$ and $q\leq q'$ then $pq\leq p' q'$,
		\item $1p=p1=p$ for all $p\in P,$
		\item $(pq)r=p(qr)$ for all $p,q,r\in P$
		\item $pq=qp$ for all $p,q\in P$ 
		\end{itemize}
\end{definition}

\noindent The main example to have in mind is the unit interval $[0,1]\subseteq\mathbb{R}$ equipped with the usual ordering $\leq$, where the monoidal product is multiplication of real numbers, and the monoidal unit is 1. In fact, $[0,1]$ can be given the structure of a closed symmetric monoidal preorder \cite[Proposition 2.1.12]{Elliott2017OnTF}, although we won't need closure here. 

\begin{definition}
Let $(\mathcal{V},\leq,\cdot, 1)$ be a symmetric monoidal preorder. A \emph{$\mathcal{V}$-enriched category} $\mathcal{C}$, or simply \emph{$\mathcal{V}$-category}, consists of the following data,
	\begin{itemize}
		\item a set $\text{ob}(\mathcal{C})$ of objects $c,d,\ldots$
		\item an element $\mathcal{C}(c,d)\in\mathcal{V}$ for every pair of objects $c$ and $d$.
	\end{itemize}
Moreover, these data must satisfy the following requirements:
	\begin{itemize}
		\item $1\leq \mathcal{C}(c,c)$ for each object $c$
		\item $\mathcal{C}(c,d)\cdot \mathcal{C}(d,e)\leq \mathcal{C}(c,e)$ for every triple of objects $c,d,e$.
	\end{itemize}
\end{definition}
\noindent There is also a notion of maps between $\mathcal{V}$-categories.

\begin{definition}Let $\mathcal{C}$ and $\mathcal{D}$ be $\mathcal{V}$-categories. A $\mathcal{V}$\emph{-functor} is a function $f\colon\text{ob}(\mathcal{C})\to\text{ob}(\mathcal{D})$ satisfying
\[\mathcal{C}(c,c')\leq \mathcal{D}(fc,fc')\]
for all objects $c,c'$ in $\mathcal{C}$.
\end{definition}

The main result below is that both language and positive semidefinite operators form $[0,1]$-categories in the desired way, and moreover there is a $[0,1]$-functor between them. The setup is the same as before. Let $X$ be a finite set, let $\pi\colon X^{N-1}\times X\to\mathbb{R}$ be any probability distribution, and let $\mathcal{L}\subseteq L\subseteq \sqcup_{k=1}^{N-1} X^k\times X$ be the subset of phrases defined in Section \ref{sec:loewner}.

\begin{proposition} The set $\mathcal{L}$ together with the following assignment for each $s,s'\in\mathcal{L}$ is a $[0,1]$-category:
\[
\mathcal{L}(s,s')=
\begin{cases}
\pi(s'|s) &\text{if $s\leq s'$}\\
0 &\text{otherwise}.
\end{cases}
\]
\end{proposition}
\begin{proof}
Observe that $1\leq \mathcal{L}(s,s)$ for each $s$, and it is simple to check that $\mathcal{L}(s,s')\mathcal{L}(s',s'')\leq \mathcal{L}(s,s'')$ for all $s,s'$ and $s''$. 
\end{proof}

\noindent With this choice of enrichment,\footnote{The notation $\pi(s'\mid s)$ is used as shorthand for the conditional probability $\pi(ts|s)$ whenever $s'=ts$ for some phrase $t$. For example, in this section we use $\pi(\text{black dog}\mid \text{dog})$ to denote the conditional probability $\pi(\text{black}\mid\text{dog})=\pi(\text{black dog})/\pi(\text{dog}).$} there is a ``morphism'' between two expressions only if one is contained in the other, and moreover that morphism is labeled with the conditional probability of containment. By a similar argument, the trace on positive semidefinite operators gives rise to a $[0,1]$-category structure on operators assigned to phrases $s\in\mathcal{L}$. In what follows, let $\mathcal{D}\subseteq \text{Pos}(V)$ denote the image of the function $\mathcal{L}\to\text{Pos}(V)$ defined by $s\mapsto\hat\rho_s$.

\begin{proposition}The set $\mathcal{D}$ together with the following assignment for each $\hat\rho_s,\hat\rho_{s'}\in\mathcal{D}$ is a $[0,1]$-category:
\[
\mathcal{D}(\hat\rho_s,\hat\rho_{s'})=
\begin{cases}
\tr\hat\rho_{s'}/\tr\hat\rho_{s} &\text{if $s\leq s'$}\\
0 &\text{otherwise}.
\end{cases}
\]
\end{proposition}
\begin{proof}
Observe that $1\leq \mathcal{D}(\hat\rho_s,\hat\rho_s)$ for each $s$, and it is simple to check that $\mathcal{D}(\hat\rho_s,\hat\rho_{s'})\mathcal{D}(\hat\rho_{s'},\hat\rho_{s''})\leq \mathcal{D}(\hat\rho_s,\hat\rho_{s''})$ for all $s,s'$ and $s''$ in $\mathcal{L}$.
\end{proof}
\noindent Recall from Proposition \ref{prop:trace} that for each $s\in\mathcal{L}$ the trace of the reduced density $\hat\rho_s$ is marginal probability, $\tr\hat\rho_s=\pi(s)$. As a result, $\tr\hat\rho_{s'}/\tr\hat\rho_{s}=\pi(s')/\pi(s)=\pi(s'|s)$ whenever $s\leq s'$ and so $\mathcal{L}(s,s')\leq \mathcal{D}(\hat\rho_s,\hat\rho_{s'})$ for all $s,s'\in\mathcal{L}$. This proves the following theorem, which can be seen as an enrichment of Proposition \ref{prop:poset_map}.

\begin{theorem}\label{theorem}
The function $\mathcal{L}\to\mathcal{D}$ defined by $s\mapsto \hat\rho_s$ is a $[0,1]$-functor.
\end{theorem}

We have therefore found a suitable home for modeling the passage from statistics in language to linear operators. To make this conclusion explicit, we quickly revisit the remarks given towards the end of Section \ref{sec:loewner}. If $s$ and $s'$ are phrases in language satisfying $s\leq s'$, then the enriched categorical structure is given by conditional probability, $\mathcal{L}(s,s')=\pi(s'|s)$. Recall from Inequality \eqref{eq:order2} that this same probability arises in the relationship between the unit-trace reduced density operators associated to each phrase, $\rho_{s}\geq \pi(s'|s)\rho_{s'}$. Theorem \ref{theorem} crystallizes the precise way in which these ideas connect. Indeed, if $s\leq s'$ then $\hat\rho_{s'}\leq\hat\rho_s$ which implies that $\mathcal{D}^{\text{op}}(\hat\rho_{s'},\hat\rho_s):=\mathcal{D}(\hat\rho_s,\hat\rho_{s'})=\tr\hat\rho_{s'}/\tr\hat\rho_{s}=\pi(s'|s)$, where the ``op'' denotes the opposite $[0,1]$-category of $\mathcal{D}.$ This recovers the intuitive idea of decorating the network of relationships between phrases in language with the appropriate conditional probabilities. The linear algebra in Section \ref{sec:text_to_density} thus prescribes a principled method of assigning text to (unnormalized) reduced density operators that preserves a simple hierarchical structure in language as well as the statistics therein. 
	\[
	\begin{tikzcd}[row sep = tiny]
		  s					\ar{dd}[swap]{\pi(s'\mid s)}
		& \hat\rho_s
		\\
		  ~					\ar[shorten >=1ex, |->, shorten <=1ex]{r}
		& ~
		\\
		  s'
		& \hat\rho_{s'}		\ar{uu}[swap]{\pi(s'\mid s)}
	\end{tikzcd}
	\]
Since the image of the functor $\mathcal{L}\to\mathcal{D}$ has more structure than its domain, we expect that operators $\rho_s$ associated to phrases $s$ carry additional information about the language, in addition to the simple form of entailment modeled here. The spectral information of $\rho_s$, for instance, may be one such source.


\subsection{Conclusion}\label{ssec:conclusion}
The hierarchy modeled in this work comes directly from the sequential structure of language. So while we can model the notion that \textit{dog} is a more general concept than \textit{small black dog}, the theory does not yet provide a way to compare phrases that aren't comparable under the preorder in $\mathcal{L}$. For example, one may not conclude that \textit{mammal} abstracts the notion of \textit{dog} since \textit{mammal} does not contain \textit{dog} as a subsequence. But considering all expressions that contain both \textit{mammal} and \textit{dog}, as in the discussion surrounding Figure \ref{fig:bias}, suggests a relationship between them. Exploring more complex hierarchies of this type  is left for future work. In this direction, we make the observation that if $\hat\rho_{s'}\leq\hat\rho_s$ then for any suffix $x,$ we have that $\langle\hat\rho_{s'}x,x\rangle=\pi(s'x)\leq \pi(sx)=\langle\hat\rho_{s}x,x\rangle$. In particular, if $\pi(s'x)>0$ for some suffix $x,$ then $\pi(sx)>0$ as well, which is to say that any valid continuation on the right of $s'$ is also a valid continuation of $s$. (In Example \ref{ex:dog1}, for instance, one sees that $\hat\rho_{\text{black cat}}\leq \hat\rho_\text{dog}$, and that the set of suffixes of \textit{black cat}, namely $\{\textit{runs}\}$, is a subset of the set of suffixes of \textit{dog}.) So while $s$ and $s'$ may not be comparable under subsequence containment, there is a clear relationship between their ``right ideals,'' to borrow from the algebraic perspective.  Understanding this algebraic connection is left for future work. There are also additional opportunities to expand the framework using ideas from category theory. For instance, the function $x\mapsto-\log(x)$ provides a mapping $[0,1]\to[0,\infty]$, suggesting that our framework may be reinterpreted in the theory of generalized metric spaces \cite{lawvere1973,Lawvere86takingcategories,willerton2013tight}. Finally, the theory proposed in this paper fits into a larger investigation of language modeling with tensor networks. One quickly notices that the dimension of the vector spaces involved grow exponentially with the size of the vocabulary $X$. Realizing and manipulating $\psi$ on a computer thus quickly becomes infeasible for real-world datasets. A similar sentiment holds for our reduced densities such as $\hat\rho_{x_{i_1}}$, which may operate on an ultra large-dimensional space. As discussed in Section \ref{sec:intro}, we propose these issues may be addressed by approximating $\psi$ by a tensor network factorization $\psi_{\text{net}}$, which should be chosen such that it can be computed efficiently, can faithfully model the statistics of language, and can easily generalize from a training corpus to unseen samples. We view the tensor network language models of \cite{PTV2017,PV2017} as source of inspiration in this direction.

\bibliographystyle{alpha}
\bibliography{references}{}
\end{document}